%% file: main.tex
\RequirePackage[svgnames]{xcolor}

\documentclass[11pt,letterpaper]{mystyle}

\usepackage[all]{hypcap}
\usepackage[svgnames]{xcolor}
\usepackage[comma,authoryear,compress]{natbib}
\hypersetup{
    colorlinks = true,
    linkcolor = ThemeBlue,
    citecolor = ThemeBlue,
    urlcolor = ThemeBlue,
}
\usepackage{iftex}
\ifPDFTeX
\else
    \usepackage{unicode-math}
\fi
\usepackage{algorithm}
\usepackage{algorithmicx}
\usepackage{algpseudocode}
\usepackage{microtype}
\usepackage{graphicx}
\usepackage{booktabs} %
\usepackage{float}
\usepackage{bigstrut}
\usepackage{booktabs}
\usepackage{amsmath}
\usepackage{amssymb}
\usepackage{mathtools}
\usepackage{amsthm}
\usepackage{nicefrac}
\usepackage{dsfont}
\usepackage{enumitem}
\usepackage{cleveref}
\usepackage{bxcoloremoji}
\usepackage{graphicx}
\usepackage{float}
\usepackage{xcolor}
\usepackage{booktabs}
\usepackage{xspace}
\usepackage[utf8]{inputenc} %
\usepackage[T1]{fontenc}    %
\usepackage{hyperref}       %
\usepackage{url}            %
\usepackage{booktabs}       %
\usepackage{amsfonts}       %
\usepackage{nicefrac}       %
\usepackage{microtype}      %
\usepackage{graphicx}
\usepackage{amssymb}
\usepackage{wrapfig}
\usepackage{lipsum}
\usepackage{enumitem}
\usepackage{stackengine}
\usepackage[font=small,labelfont=bf]{caption}
\usepackage{color}
\usepackage{adjustbox}
\usepackage{lipsum}

\usepackage{rotating}
\usepackage{makecell}
\usepackage{multirow}
\usepackage{booktabs}
\usepackage[table]{xcolor}

\input{macro}

\newtcolorbox{AIbox}[2][]{aibox,title=#2,#1}
\definecolor{lightblue}{rgb}{0.22,0.45,0.70}%
\definecolor{Gray}{gray}{0.95}
\definecolor{Cornsilk}{rgb}{1.0, 0.97, 0.86}

\usepackage{amsmath}

\usepackage[all]{hypcap}

\title{\ours: Deriving with Policy Optimization, Training with Self-Distillation}

\runningtitle{\ours: 
Deriving with Policy Optimization, Training with Self-Distillation}

\author[1,*]{Jiawei Xu}
\author[1,*]{Minghui Liu}
\author[1]{Juzheng Zhang}
\author[1]{Tom Goldstein}
\author[1]{Furong Huang}

\affil[1]{University of Maryland, College Park}
\equalcontribution{* Equal contribution and joint project leadership.}

\correspondingauthor{Furong Huang \href{https://furong-huang.com}{https://furong-huang.com}; Email \href{mailto:furongh@umd.edu}{furongh@umd.edu}}

\begin{document}

\input{sections/abstract}

\begingroup
\hbadness=10000
\maketitle
\endgroup
\vspace{2mm}

\input{figures/main_plot}
\input{sections/introduction}
\input{sections/method}
\input{sections/relatedwork}
\input{sections/experiments}
% \clearpage
\input{sections/conclusion}
\input{sections/acknowledge}
\clearpage
\bibliography{main}

\appendix
\input{sections/appendix}
\end{document}

%% file: macro.tex
\newcommand{\x}{\mathbf{x}}

\newcommand{\ours}{\textsc{$\beta$-OPSD}}
\newcommand{\oursbold}{\textbf{$\boldsymbol{\beta}$-OPSD}\xspace}

\newcommand{\posgain}[1]{{\scriptsize\textcolor{green!50!black}{(#1)}}}
\newcommand{\neggain}[1]{{\scriptsize\textcolor{red!70!black}{(#1)}}}

\definecolor{blanchedalmond}{rgb}{1.0, 0.92, 0.8}
\definecolor{carmine}{rgb}{0.59, 0.0, 0.09}
\definecolor{lightblue}{rgb}{0.22,0.45,0.70}%

\newtheorem{theorem}{Theorem}[section]

\newtheorem{proposition}[theorem]{Proposition}

\newtheorem{remark}[theorem]{Remark}

\renewcommand{\mathbf}{\boldsymbol}

\makeatletter
\def\Ddots{\mathinner{\mkern1mu\raise\p@
\vbox{\kern7\p@\hbox{.}}\mkern2mu
\raise4\p@\hbox{.}\mkern2mu\raise7\p@\hbox{.}\mkern1mu}}
\makeatother

\definecolor{amaranth}{rgb}{0.9, 0.17, 0.31}
\definecolor{antiquebrass}{rgb}{0.8, 0.58, 0.46}
\definecolor{antiquefuchsia}{rgb}{0.57, 0.36, 0.51}
\definecolor{chromeyellow}{rgb}{0.31, 0.47, 0.26}

%% file: sections/abstract.tex
\begin{abstract}

On-policy self-distillation (OPSD) is a promising approach to improve reasoning language
models, but it remains brittle in practice: making it work reliably often requires substantial
engineering effort. We identify a structural source of this difficulty: vanilla OPSD is
precisely the $\beta=1$ member of a broader policy-optimization family, where $\beta$ weights
the KL penalty anchoring the student to a reference policy. This equivalence turns $\beta$ from an implicit value
fixed at one into a controllable regularization parameter, yielding a
more general formulation that trades off proximity to a reference policy against privileged
teacher guidance. We introduce \oursbold and derive its optimal policy as a geometric
interpolation between the reference policy and the privileged teacher. Directly optimizing
this objective with reinforcement learning, however, would be costly and high-variance.
Rather than optimize the RL objective directly, we turn its closed-form solution into a
distillation target. Each value of $\beta$ selects a target along the reference-to-teacher
path, which we implement efficiently by mixing their token-level logits. In this way,
inexpensive distillation approximates the solution of
expensive policy optimization. Return-to-go credit assignment further aligns token updates
with the sequence-level objective while retaining the simplicity of OPSD. Experiments on
mathematical reasoning benchmarks show that \ours{} consistently
outperforms vanilla OPSD, improving optimization stability and downstream reasoning
performance. Our results provide a principled route from self-distillation to policy
optimization and back---without sacrificing the efficiency that makes OPSD practical.

\end{abstract}

%% file: figures/main_plot.tex
\begin{figure*}[htbp]
    \centering
    \includegraphics[width=\textwidth,trim={0 100bp 25bp 0},clip]{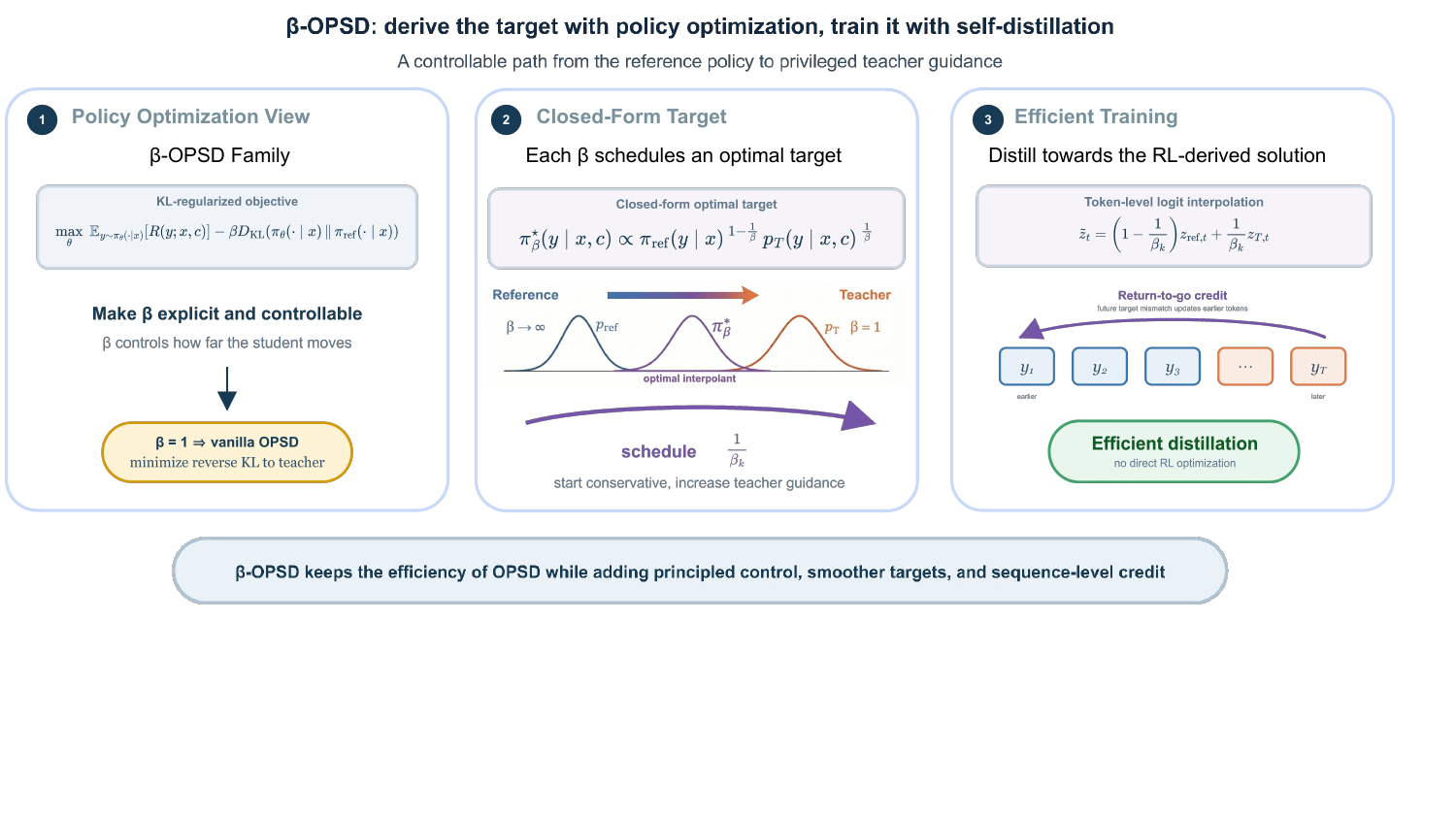}
    \caption{\textbf{Overview of \oursbold{}}. Vanilla OPSD is the $\beta=1$ case of a broader KL-regularized objective, whose optimal policy is a geometric interpolant between the reference policy and the privileged teacher. $\beta$-OPSD turns this optimal policy into a practical distillation target through scheduled logit interpolation, while retaining efficient token-level training with return-to-go credit assignment.
    }
    \label{fig:main_plot}
\end{figure*}

%% file: sections/introduction.tex
\vspace{-4mm}
\section{Introduction}

Reasoning performance of large language models has advanced rapidly through reinforcement
learning and self-improvement techniques
\citep{shao2024deepseekmath, ouyang2022training, guo2025deepseek}. Among these approaches,
on-policy self-distillation (OPSD) is a promising way to improve reasoning: the student learns
from its own sampled trajectories while a teacher supplies supervision using privileged
information available only during training, such as reference solutions, verified reasoning
traces, or external feedback
\citep{zhao2026self, shenfeld2026self, hubotter2026reinforcement, li2026rethinking,
jang2026stable, yang2026self}. Unlike offline distillation, OPSD lets the teacher respond to
the student's actual intermediate states. Yet this promise has not produced a reliably
effective recipe; in practice, OPSD can be brittle and often demands substantial engineering.

Most existing OPSD methods begin from direct teacher imitation. Given a student-generated
trajectory, the privileged teacher provides token-level predictions, and the student minimizes
a reverse Kullback--Leibler (KL) divergence to those predictions
\citep{zhao2026self, jang2026stable, li2026rethinking, yang2026self}. This formulation fixes
the teacher as the unique optimization target and provides no explicit control over how
aggressively the student should move away from its current or initial policy. Is direct teacher
matching truly the only principled OPSD objective?

Our key observation is that it is not. Vanilla OPSD is precisely the $\beta=1$ member of a
broader KL-regularized policy-optimization family, where $\beta$ weights the KL penalty that
anchors the learned policy to a reference policy. This equivalence turns $\beta$ from an
implicit value fixed at one into an explicit control over the strength of teacher guidance.
It also changes the interpretation of OPSD: rather than a single teacher-imitation objective,
it is one point in a continuum of policy-optimization problems that trade off improvement
toward the privileged teacher against proximity to the reference policy.

This perspective leads to \oursbold, a generalized family of on-policy self-distillation
objectives parameterized by $\beta$. Vanilla OPSD is recovered exactly at $\beta=1$, while
larger values yield more conservative updates that remain closer to the reference policy.
The family therefore exposes a degree of freedom that vanilla OPSD leaves fixed and offers
a principled mechanism for addressing its brittleness: instead of moving abruptly toward
the teacher, the target can progress smoothly from the reference policy toward privileged
teacher guidance over the course of training.

The policy-optimization view also tells us what these targets should be. We derive the
optimal policy of \ours{} in closed form and show that it is a geometric interpolation
between the reference policy and the privileged teacher. Directly optimizing the resulting
RL objective, however, would incur costly rollouts, advantage estimation, and high-variance
gradients; the exact sequence-level optimum also contains an intractable normalizer. We
therefore use policy optimization for the derivation and return to distillation for training.
Each value of $\beta$ defines an RL-derived target along the reference-to-teacher path, which
we realize efficiently through token-level logit interpolation. Scheduling $\beta$ then gives
a smooth curriculum of targets while preserving the computational structure of OPSD. In
short, inexpensive distillation approximates the solution of expensive policy optimization.

Finally, matching the right target still requires assigning credit across the generated
trajectory. Standard token-local OPSD updates are myopic: an early token changes later
prefixes and therefore affects future student--target mismatch. A policy-gradient derivation
of the sequence-level objective yields return-to-go credit assignment, which propagates this
future mismatch back to earlier tokens. Combined with the interpolated target, this gives a
practical algorithm that retains token-level distillation while more faithfully optimizing the
underlying sequence-level objective.

We evaluate \ours{} on mathematical reasoning benchmarks using open-source reasoning
language models ranging from 1.7B to 8B parameters. Across multiple benchmarks, \ours{}
consistently improves optimization stability and reasoning performance over vanilla OPSD.
On Qwen3-1.7B \citep{yang2025qwen3}, it improves avg@12 by up to 9.16 percentage points
over vanilla OPSD across AIME 2024 \citep{aime24}, AIME 2025 \citep{aime25}, and HMMT
2025 \citep{dekoninck2026matharena}. Ablations further isolate the benefits of both the
RL-derived interpolation target and sequence-level return-to-go credit assignment.

Our contributions are summarized as follows: 
\begin{itemize} 
    \item We show that vanilla OPSD is precisely the $\beta=1$ member of a broader
    KL-regularized policy-optimization family, making its implicit regularization choice
    explicit and controllable.
    \item We introduce \oursbold and derive its closed-form optimal policy as a geometric
    interpolation between the reference policy and the privileged teacher.
    \item We turn the RL-derived solution back into an efficient distillation algorithm using
    scheduled logit interpolation and return-to-go credit assignment.
    \item We demonstrate consistent improvements over vanilla OPSD on mathematical
    reasoning benchmarks and validate the contribution of both practical components.
\end{itemize}

\label{sec:intro}
\vspace{-2mm}

%% file: sections/method.tex
\section{Method}
\label{sec:method}

This section develops \ours{} following the paper's central principle: derive the target
with policy optimization, then train it with self-distillation. We first show that vanilla OPSD
is precisely the $\beta=1$ member of a broader KL-regularized policy-optimization family,
making the strength of reference regularization explicit and controllable. We then derive the
family's optimal policy as a geometric interpolation between the reference policy and the
privileged teacher. Because direct RL optimization is costly and the exact sequence-level
target is intractable, we return to efficient distillation by realizing the target through
scheduled token-level logit interpolation. Finally, we derive return-to-go credit assignment
to align the practical token-level update with the underlying sequence-level objective.
Appendix~\ref{app:notation} provides a consolidated reference for the notation used
throughout the paper.

\subsection{Vanilla OPSD as \texorpdfstring{$\beta=1$}{beta=1} Policy Optimization}
\label{sec:beta_opsd}

Let $\pi_\theta(y\mid x)$ denote the student policy, $p_T(y\mid x,c)$ the privileged teacher,
and $\pi_{\mathrm{ref}}(y\mid x)$ a reference policy, such as the initial student or a
stop-gradient copy of the current student. We start from the standard KL-regularized RL
objective
\begin{equation}
    \max_{\theta}
    \;
    \mathbb{E}_{y\sim\pi_\theta(\cdot\mid x)}
    \left[
    R(y;x,c)
    \right]
    -
    \beta
    D_{\mathrm{KL}}
    \left(
    \pi_\theta(\cdot\mid x)
    \,\middle\|\,
    \pi_{\mathrm{ref}}(\cdot\mid x)
    \right),
    \label{eq:kl_regularized_rl}
\end{equation}
where $\beta>0$ controls how strongly the learned policy is constrained to remain close to
the reference.

To embed standard on-policy self-distillation within this policy-optimization family, we
choose the reward to be the teacher-to-reference log-ratio:
\begin{equation}
    R(y;x,c)
    =
    \log
    \frac{
    p_T(y\mid x,c)
    }{
    \pi_{\mathrm{ref}}(y\mid x)
    }.
    \label{eq:teacher_log_ratio_reward}
\end{equation}

This reward is high when a trajectory is more likely under the privileged teacher than under
the reference policy. Substituting Eq.~\ref{eq:teacher_log_ratio_reward} into
Eq.~\ref{eq:kl_regularized_rl} gives the family
\begin{equation}
    \mathcal{J}_{\beta}(\theta)
    =
    \mathbb{E}_{y\sim\pi_\theta(\cdot\mid x)}
    \left[
    \log
    \frac{
    p_T(y\mid x,c)
    }{
    \pi_{\mathrm{ref}}(y\mid x)
    }
    \right]
    -
    \beta
    D_{\mathrm{KL}}
    \left(
    \pi_\theta(\cdot\mid x)
    \,\middle\|\,
    \pi_{\mathrm{ref}}(\cdot\mid x)
    \right).
    \label{eq:beta_opsd_objective}
\end{equation}
We call this objective \emph{$\beta$-OPSD}. It generalizes standard on-policy
self-distillation by treating $\beta$ as a variable regularization parameter rather than
implicitly fixing it at one.
The log-ratio reward favors trajectories preferred by the privileged teacher over the
reference policy, while the KL term penalizes movement away from the reference; $\beta$
controls the trade-off between these two forces.

\begin{proposition}[OPSD is the $\beta=1$ case]
\label{prop:opsd_beta_one}
When $\beta=1$, maximizing Eq.~\ref{eq:beta_opsd_objective} is equivalent to minimizing
the standard sequence-level reverse KL,
\begin{equation}
    D_{\mathrm{KL}}
    \left(
    \pi_\theta(\cdot\mid x)
    \,\middle\|\,
    p_T(\cdot\mid x,c)
    \right).
    \label{eq:standard_opsd_kl}
\end{equation}
\end{proposition}

At $\beta=1$, the reference-policy terms in the log-ratio reward and KL penalty cancel,
leaving exactly the negative reverse KL from the student to the privileged teacher. Thus,
standard OPSD is not separate from policy optimization; it is one member of this
KL-regularized family. For $\beta>1$, the remaining reference penalty makes the update
more conservative, while decreasing $\beta$ toward $1$ moves the objective more
aggressively toward the teacher. The proof is given in Appendix~\ref{app:regularized_rl}.

\subsection{Variable \texorpdfstring{$\beta$}{beta} Defines a Path of Optimal Targets}
\label{sec:optimal_policy_interpolant}

Once vanilla OPSD is recognized as the $\beta=1$ case, the next question is what target
each value of $\beta$ defines. For any fixed regularization coefficient, the optimal policy of
$\beta$-OPSD can be written in closed form.

\begin{proposition}[Optimal policy of $\beta$-OPSD]
\label{prop:geometric_optimal_policy}
Assuming compatible support, the maximizer of Eq.~\ref{eq:beta_opsd_objective} is
\begin{equation}
    \pi_{\beta}^{\star}(y\mid x,c)
    =
    \frac{
    \pi_{\mathrm{ref}}(y\mid x)^{1-\frac{1}{\beta}}
    p_T(y\mid x,c)^{\frac{1}{\beta}}
    }{
    Z_{\beta}(x,c)
    },
    \label{eq:sequence_geometric_mixture}
\end{equation}
where $Z_{\beta}(x,c)$ normalizes over full trajectories.
\end{proposition}

The proof is in Appendix~\ref{app:geometric-optimal-policy}. The solution is a geometric
interpolation between the reference policy and the privileged teacher. For $\beta\ge 1$, the
teacher weight $\frac{1}{\beta}$ lies in $[0,1]$: $\beta=1$ recovers the teacher endpoint,
while $\beta\to\infty$ approaches the reference endpoint. Thus, the teacher is no longer
the unique target. Each value of $\beta$ selects a different optimal policy along a principled
reference-to-teacher path.

\paragraph{Scheduling the interpolation.}
The theory above defines a family of targets for fixed $\beta$. In practice, we exploit this
new degree of freedom by varying the coefficient over training. We write the coefficient at
step $k$ as $\beta_k$; its inverse directly determines the teacher weight. We require
\begin{equation}
    \frac{1}{\beta_k}\in[0,1],
    \qquad
    \text{equivalently } \beta_k\ge 1,
\end{equation}
so that the target remains on the path between the reference and the teacher. A larger
$\frac{1}{\beta_k}$ gives a more teacher-like target, while a smaller
$\frac{1}{\beta_k}$ keeps the target closer to the reference policy.

In practice, we use a bounded linear schedule for the teacher weight:
\begin{equation}
    \frac{1}{\beta_k}
    =
    w_{\mathrm{start}}
    +
    \left(
    w_{\mathrm{end}}-w_{\mathrm{start}}
    \right)
    \frac{k}{K-1},
    \qquad
    k=0,\ldots,K-1.
    \label{eq:bounded_beta_schedule}
\end{equation}
Equivalently,
\begin{equation}
    \beta_k
    =
    \left[
    w_{\mathrm{start}}
    +
    \left(
    w_{\mathrm{end}}-w_{\mathrm{start}}
    \right)
    \frac{k}{K-1}
    \right]^{-1}.
\end{equation}
When $w_{\mathrm{start}}<w_{\mathrm{end}}$, early targets remain closer to the
reference policy, while later targets incorporate stronger teacher guidance. This smooth
curriculum avoids an abrupt projection onto the privileged teacher and provides a direct
mechanism for reducing the brittleness of vanilla OPSD.

\paragraph{Local logit realization.}
Although Eq.~\ref{eq:sequence_geometric_mixture} gives the optimal trajectory-level
distribution, it is not directly usable for language models because the normalizer
$Z_{\beta}(x,c)$ ranges over all possible completions. We therefore approximate the
optimal policy locally at each decoding prefix.

Let
\begin{equation}
    h_t=(x,y_{<t})
\end{equation}
be the prefix at token $t$. Let $z_{\mathrm{ref}}(\cdot\mid h_t)$ denote the reference
or student-side logits, and let $z_T(\cdot\mid h_t,c)$ denote the privileged teacher logits.
Given the scheduled coefficient $\beta_k$, we define the local logit-interpolant target as
\begin{equation}
    \tilde p_{\beta_k}(\cdot\mid h_t,c)
    =
    \operatorname{softmax}
    \left(
    \left(1-\frac{1}{\beta_k}\right)
    z_{\mathrm{ref}}(\cdot\mid h_t)
    +
    \frac{1}{\beta_k}
    z_T(\cdot\mid h_t,c)
    \right).
    \label{eq:local_logit_interpolant}
\end{equation}
At each prefix, this distribution is exactly the normalized geometric interpolation of the
local reference and teacher distributions: multiplying probabilities geometrically is
equivalent, up to normalization, to linearly interpolating their logits.

The resulting autoregressive interpolant target is
\begin{equation}
    \tilde p_{\beta_k}(y\mid x,c)
    =
    \prod_{t=1}^{T}
    \tilde p_{\beta_k}(y_t\mid x,y_{<t},c).
    \label{eq:sequence_logit_interpolant}
\end{equation}
Because normalization is applied separately at each prefix, the resulting autoregressive
distribution is a tractable local approximation to the sequence-level optimum in
Eq.~\ref{eq:sequence_geometric_mixture}. This is the target used by the practical
distillation objective below.

\subsection{From Policy Optimization Back to Distillation}
%\subsection{Policy-Optimal Targets via Efficient Self-Distillation}
\label{sec:logit_interpolant_distillation}

The $\beta$-OPSD objective in Eq.~\ref{eq:beta_opsd_objective} can in principle be
optimized with KL-regularized RL algorithms. Doing so, however, would introduce additional
RL machinery, such as reward-weighted policy updates and advantage estimation, together
with higher-variance gradients. The analysis above provides a cheaper route: the policy
objective specifies an optimal sequence-level target, and
Eq.~\ref{eq:sequence_logit_interpolant} gives its tractable local approximation.

This is the main ``back to distillation'' step. Rather than directly maximizing the RL
objective, at training step $k$ we train the student to match the RL-derived interpolant
target $\tilde p_{\beta_k}$ by minimizing
\begin{equation}
    \mathcal{L}_{\beta_k}(\theta)
    =
    D_{\mathrm{KL}}
    \left(
    \pi_\theta(\cdot\mid x)
    \,\middle\|\,
    \tilde p_{\beta_k}(\cdot\mid x,c)
    \right).
    \label{eq:interpolant_distillation_objective}
\end{equation}
Equivalently,
\begin{equation}
    \mathcal{L}_{\beta_k}(\theta)
    =
    \mathbb{E}_{y\sim\pi_\theta(\cdot\mid x)}
    \left[
    \log
    \frac{
    \pi_\theta(y\mid x)
    }{
    \tilde p_{\beta_k}(y\mid x,c)
    }
    \right].
\end{equation}
Using the autoregressive factorization in Eq.~\ref{eq:sequence_logit_interpolant}, this
becomes
\begin{equation}
    \mathcal{L}_{\beta_k}(\theta)
    =
    \mathbb{E}_{y\sim\pi_\theta(\cdot\mid x)}
    \left[
    \sum_{t=1}^{T}
    \rho_t^{\beta_k}(y)
    \right],
    \label{eq:interpolant_kl_decomposition}
\end{equation}
where
\begin{equation}
    \rho_t^{\beta_k}(y)
    =
    \log \pi_\theta(y_t\mid x,y_{<t})
    -
    \log \tilde p_{\beta_k}(y_t\mid x,y_{<t},c).
    \label{eq:interpolant_log_ratio}
\end{equation}

This objective has the same form as vanilla OPSD, except that the fixed teacher
$p_T$ is replaced by the scheduled interpolant target $\tilde p_{\beta_k}$. When
$\beta_k=1$, the interpolant target reduces to the privileged teacher, and
Eq.~\ref{eq:interpolant_distillation_objective} recovers standard OPSD. For larger
$\beta_k$, the target remains closer to the reference policy, yielding a more conservative
distillation objective.

The interpolant target is efficient to compute. Once the reference/student logits and
teacher logits are available, constructing $\tilde p_{\beta_k}$ only requires an elementwise
weighted sum of the two logit vectors followed by a softmax, as in
Eq.~\ref{eq:local_logit_interpolant}. Importantly, this avoids the intractable
sequence-level normalizer $Z_\beta(x,c)$ in Eq.~\ref{eq:sequence_geometric_mixture}.
Thus, policy optimization specifies where the student should move, while token-level
distillation provides an inexpensive way to get there.

\subsection{Practical \texorpdfstring{$\beta$}{beta}-OPSD with Return-to-Go Credit}
\label{sec:practical_slide_objective}
We now describe the practical training loss. At each update, we sample trajectories
on-policy from the current student:
\begin{equation}
    y\sim\pi_\theta(\cdot\mid x).
\end{equation}
Along the sampled trajectory, we evaluate the token-level mismatch between the student and
the interpolant target:
\begin{equation}
    \rho_t^{\beta_k}(y)
    =
    \log \pi_\theta(y_t\mid x,y_{<t})
    -
    \log \tilde p_{\beta_k}(y_t\mid x,y_{<t},c).
    \label{eq:practical_interpolant_log_ratio}
\end{equation}
The scalar estimator $\sum_t \rho_t^{\beta_k}(y)$ is an unbiased Monte Carlo estimate of
the sequence-level reverse KL in Eq.~\ref{eq:interpolant_distillation_objective}. However,
using each token's instantaneous log-ratio as its training weight gives a myopic gradient:
it ignores how an earlier token affects future prefixes and future target mismatch.

Therefore, $\beta$-OPSD uses a return-to-go weight:
\begin{equation}
    G_{t,\gamma}^{\beta_k}(y)
    =
    \sum_{s=t}^{T}
    \gamma^{s-t}
    \rho_s^{\beta_k}(y),
    \qquad
    \gamma\in[0,1].
    \label{eq:slide_return_to_go}
\end{equation}
When $\gamma=1$, this recovers the exact sequence-level gradient estimator; in practice,
we use $\gamma<1$ to control the magnitude of the return on long generations. The proof is in Appendix~\ref{app:onpolicy-lookahead-gradient}.

The practical $\beta$-OPSD loss is
\begin{equation}
    \mathcal{L}_{\beta}(\theta;y)
    =
    \frac{1}{T}
    \sum_{t=1}^{T}
    \operatorname{sg}
    \left(
    G_{t,\gamma}^{\beta_k}(y)
    \right)
    \log \pi_\theta(y_t\mid x,y_{<t}),
    \label{eq:slide_loss}
\end{equation}
where $\operatorname{sg}(\cdot)$ denotes stop-gradient. 
The interpolant logits and
return-to-go weights are detached; the only gradient path is through the student
log-probability. Thus, $\beta$-OPSD differs from vanilla OPSD in two simple ways:
\begin{enumerate}
    \item it replaces the fixed teacher target with a scheduled logit-interpolant target, and
    \item it replaces local token weights with return-to-go weights for sequence-level credit
    assignment.
\end{enumerate}

The first modification is the main conceptual contribution: it follows from the
$\beta$-OPSD optimal policy. The second is a practical sequence-level estimator that makes
the distillation loss faithful to the trajectory objective.

%% file: sections/relatedwork.tex
\section{Related Work}
\label{sec:related_work}

The method above connects two themes that are often treated separately: how to choose a stable teacher target, and how to assign credit along an on-policy trajectory. We organize related work around this split. Prior work provides ingredients for smoother supervision, reference-regularized policy optimization, and return-based credit assignment; \ours{} combines these ingredients into a single distillation procedure.

\paragraph{On-policy distillation and self-distillation.} On-policy distillation trains the student on its own generated trajectories, using dense
teacher feedback at each token. This avoids the exposure-bias mismatch of purely offline
distillation, but introduces a new challenge: the teacher must provide reliable supervision
on student-generated prefixes, which may be outside the teacher's typical distribution \citep{li2026rethinking, jang2026stable}.
Recent self-distillation methods further instantiate the teacher as the same model with
access to privileged information, such as ground-truth solutions or feedback \citep{zhao2026self, liu2026self, yang2026self, shen2026anti, zhang2026opsdl, hubotter2026reinforcement, he2026self}. \ours{} belongs to this line of work, but differs by replacing the fixed teacher target with a scheduled
student--teacher interpolant and by correcting the myopic local token-KL gradient with a
Look-Ahead return.

\paragraph{RLHF and reference-regularized policy optimization.}
Our theoretical formulation is closely related to KL-regularized reinforcement learning and
RLHF, where a policy is optimized to maximize reward while staying close to a reference
model \citep{ziegler2019fine, ouyang2022training, schulman2017proximal}. DPO shows that the optimal policy of
a KL-regularized preference objective can be written in closed form, revealing a direct
connection between rewards, reference policies, and optimal distributions. We use a similar
perspective to reinterpret on-policy distillation: the teacher-to-reference log-probability
ratio acts as a reward, and standard OPD is recovered as the $\beta=1$ special case. This
view leads directly to the geometric interpolant between the reference student and the
privileged teacher.

\paragraph{Interpolation, curricula, and teacher-guided learning.}
Curriculum learning aims to make optimization smoother by gradually changing the
difficulty or source of supervision \citep{bengio2009curriculum}. Related ideas also
appear in sequence prediction and imitation learning: scheduled sampling gradually moves
training from teacher-forced inputs toward model-generated inputs
\citep{bengio2015scheduled}, while DAgger addresses distribution shift by querying expert
supervision on learner-induced states \citep{ross2011reduction}. More recent work has
explored student--teacher mixture reformulations to bridge the distributional gap in
distillation \citep{jang2026stable}, as well as teacher-guided or
student--teacher cooperative sampling for improving training rollouts
\citep{xu2025speculative,shenfeld2023tgrl}. Our logit-interpolant target can be viewed as
a distributional curriculum: early targets remain close to the student, while later targets
move toward the privileged teacher. Unlike heuristic interpolation or guidance used only
for sampling, our interpolant is derived as the optimal policy of a reference-regularized
distillation objective. This provides a principled target path from student to teacher and,
when combined with return-to-go credit assignment, yields a stable and future-aware
on-policy distillation algorithm.

\paragraph{Policy gradients and credit assignment.}
The return-to-go estimator used in $\beta$-OPSD is closely related to classical policy-gradient methods such as
REINFORCE \citep{williams1992simple}, where unbiased gradient estimation assigns each action its
return-to-go rather than only its instantaneous reward. In OPD, the analogous per-step
signal is the teacher--student log-ratio along a generated trajectory. Prior work has explored
using local token-level signals for stable distillation \citep{li2026rethinking, armandpour2026unmasking, shen2026anti, jang2026stable, kim2026does}, and related derivations have
also identified unbiased sequence-level gradient estimators for reverse-KL objectives
\citep{agrawal2026reinforcement, liu2026self}. Future-aware credit assignment has also been studied in GRPO-style optimization
\citep{schulman2015high, sutton1998reinforcement}. Our work builds on this perspective and shows that, for on-policy distillation, the
standard local token-KL update is myopic because it ignores the future distributional
mismatch induced by earlier tokens. \ours{} applies return-to-go credit assignment directly
to teacher-guided distillation: the exact undiscounted Look-Ahead estimator recovers an
unbiased sequence-level KL gradient under student on-policy sampling, while discounted
implementations provide a practical stability--variance trade-off for long generations.

%% file: sections/experiments.tex
\vspace{-2mm}
\section{Experiments}
\label{sec:experiments}

\subsection{Experimental Setup}
\label{sec:experimental_setup}

\paragraph{Models and Datasets.}
We evaluate $\beta$-OPSD on the instruct-tuned Qwen3 model family, including
Qwen3-1.7B, Qwen3-4B, and Qwen3-8B. For training, we use the mathematical reasoning
subset of OpenThoughts. Each training example contains a problem $x$ and a ground-truth
solution, which we use as privileged information $c$ for the teacher branch. The student
generates solutions conditioned only on the problem $x$, while the privileged teacher can
condition on both $x$ and $c$ during training.

We evaluate on competition-level mathematical reasoning benchmarks, including
AIME 2024, AIME 2025, and HMMT 2025. These benchmarks require multi-step reasoning
and therefore test whether the method improves the quality of generated reasoning
trajectories rather than only local token imitation.

\paragraph{Baselines.}
We compare against three classes of baselines. First, we include the base model without
post-training. Second, we compare with supervised fine-tuning (SFT), which trains directly
on ground-truth solutions with the standard next-token prediction objective. Third, we
compare with vanilla on-policy self-distillation (OPSD), which distills directly toward the
privileged teacher using a local token-KL objective. In our notation, vanilla OPSD uses the
teacher endpoint as the target but does not use the scheduled interpolant curriculum or the
return-to-go credit-assignment estimator. We also compare with GRPO-style reinforcement
learning baselines, which optimize the student from sampled generations using
outcome-level rewards.

\paragraph{Implementation Details.}
For the main $\beta$-OPSD experiments, we use a bounded linear path for the teacher
weight rather than starting exactly at the student endpoint or ending exactly at the teacher
endpoint. Let
\begin{equation}
    s_k = \frac{k}{K-1},
    \qquad
    k=0,\ldots,K-1,
\end{equation}
be the normalized training progress. We define the practical teacher weight
\begin{equation}
    w_k
    =
    w_{\mathrm{start}}
    +
    \left(
    w_{\mathrm{end}}-w_{\mathrm{start}}
    \right)s_k,
    \label{eq:exp_weight_schedule}
\end{equation}
where $w_k=1/\beta_k$. Thus, $w_k$ is the teacher-side weight in the logit interpolant,
while $\beta_k$ is the corresponding regularization coefficient in the $\beta$-OPSD view.

At each update, we construct the interpolant target from a stop-gradient copy of the
current student and the privileged teacher. For a prefix $h_t=(x,y_{<t})$, we use
\begin{equation}
    \tilde p_{\beta_k}(\cdot\mid h_t,c)
    =
    \operatorname{softmax}
    \left(
    (1-w_k)z_{\bar{\theta}}(\cdot\mid h_t)
    +
    w_k z_T(\cdot\mid h_t,c)
    \right),
    \label{eq:exp_interpolant_target}
\end{equation}
where $\bar{\theta}=\operatorname{sg}(\theta)$ denotes a stop-gradient copy of the
current student. The token-level mismatch is
\begin{equation}
    \rho_t^{\beta_k}(y)
    =
    \log \pi_\theta(y_t\mid x,y_{<t})
    -
    \log \tilde p_{\beta_k}(y_t\mid x,y_{<t},c).
    \label{eq:exp_token_mismatch}
\end{equation}
We then compute the discounted return-to-go in Eq.~\ref{eq:slide_return_to_go} and the \ours{} loss in Eq.~\ref{eq:slide_loss}
\begin{equation}
    G_{t,\gamma}^{\beta_k}(y)
    =
    \sum_{s=t}^{T}
    \gamma^{s-t}
    \rho_s^{\beta_k}(y),
    \label{eq:exp_return_to_go}
\end{equation}
and optimize the practical $\beta$-OPSD loss
\begin{equation}
    \mathcal{L}_{\beta\text{-OPSD}}(\theta;y)
    =
    \frac{1}{T}
    \sum_{t=1}^{T}
    \operatorname{sg}
    \left(
    G_{t,\gamma}^{\beta_k}(y)
    \right)
    \log \pi_\theta(y_t\mid x,y_{<t}).
    \label{eq:exp_beta_opsd_loss}
\end{equation}

Unless otherwise specified, we use $w_{\mathrm{start}}=0.5$,
$w_{\mathrm{end}}=0.8$, $K=200$, and $\gamma=0.99$. These choices keep early targets
relatively close to the student while still exposing the model to stronger teacher guidance
later in training. All experiments are conducted on NVIDIA RTX A6000 or NVIDIA H200
GPUs with LoRA \citep{hu2022lora}. Additional training, generation, and evaluation details
are provided in Appendix~\ref{sec:exp_details}.

\input{tables/qwen1.7b_eval}
\input{sections/algorithm}

\subsection{Results on Math Reasoning}
\label{sec:math_results}

Table~\ref{tab:qwen3_results} reports avg@12 performance for Qwen3 models evaluated
at the 100-step checkpoint from the default 200-step training schedule. Across all three
model scales, $\beta$-OPSD improves the overall average over Vanilla OPSD and achieves
the best average performance among all evaluated methods. The gains are most pronounced
for Qwen3-1.7B, where $\beta$-OPSD improves over Vanilla OPSD by $9.16$, $5.27$,
and $2.78$ points on AIME 2024, AIME 2025, and HMMT 2025, respectively, yielding a
$5.74$-point improvement in average performance.

The larger-scale results show that the method remains competitive as model size increases.
On Qwen3-4B and Qwen3-8B,
$\beta$-OPSD improves the overall average over Vanilla OPSD by $1.76$ and $1.66$
points, respectively. While it slightly underperforms Vanilla OPSD on some individual
benchmarks, such as AIME 2024 for Qwen3-4B and HMMT 2025 for Qwen3-8B, it gives
the strongest overall performance across the benchmark suite.

Compared with other post-training baselines, $\beta$-OPSD consistently outperforms SFT
across all model scales and achieves stronger performance than GRPO. Overall, these results support our central claim: replacing direct teacher imitation with
scheduled logit-interpolant targets provides a smoother and more effective distillation
objective. The advantage is especially clear for smaller models, where the student--teacher
distributional gap is larger, and remains competitive or stronger at larger
scales.

\subsection{Ablation Studies}
\label{sec:ablation_studies}

We conduct controlled ablations to isolate the contribution of each component of
$\beta$-OPSD. Unless otherwise specified, all variants use the same student, privileged
teacher, training data, optimization hyperparameters, training budget, and evaluation
protocol as the main experiments. We vary only the target distribution, credit-assignment
estimator, interpolant reference, or interpolation schedule under study. All ablations are
conducted with Qwen3-1.7B and report avg@12 on AIME 2024, AIME 2025, and HMMT 2025.

\subsubsection{Effect of Logit Interpolants}
\label{sec:ablation_logit_interpolants}

We first study whether the improvement of $\beta$-OPSD comes from the logit-interpolant
target or solely from the return-to-go credit-assignment estimator. To isolate the effect of
the target distribution, we apply the same return-to-go loss to two different distillation
targets.

The first variant, \emph{Teacher + Return-to-go}, distills directly toward the privileged
teacher. Its token-level log-ratio is
\begin{equation}
    \rho_t^{\mathrm{teach}}(y)
    =
    \log \pi_\theta(y_t\mid x,y_{<t})
    -
    \log p_T(y_t\mid x,y_{<t},c),
\end{equation}
with return-to-go $G_{t,\gamma}^{\mathrm{teach}}(y)=\sum_{s=t}^{T}\gamma^{s-t}\rho_s^{\mathrm{teach}}(y).$

The second variant uses the scheduled $\beta$-OPSD target. Its token-level log-ratio is 
\begin{equation}
    \rho_t^{\beta_k}(y)
    =
    \log \pi_\theta(y_t\mid x,y_{<t})
    -
    \log \tilde p_{\beta_k}(y_t\mid x,y_{<t},c),
\end{equation}
with return-to-go $G_{t,\gamma}^{\beta_k}(y)=\sum_{s=t}^{T}\gamma^{s-t}\rho_s^{\beta_k}(y).$

Both variants use student on-policy trajectories and identical return-to-go credit
assignment. Therefore, this comparison isolates the benefit of replacing the fixed teacher
target with the scheduled logit-interpolant target.

As shown in Table~\ref{tab:logit_interpolant_ablation}, the $\beta$-OPSD target
consistently outperforms direct teacher distillation, improving avg@12 by $6.03$ points on
AIME 2024, $5.30$ points on AIME 2025, and $1.67$ points on HMMT 2025. These gains
highlight the importance of a smooth transition in the target objective: rather than abruptly
projecting the student onto the teacher, the interpolant target reduces the student--target
gap and provides a more stable learning signal.

\input{tables/logit_interpolant_ablation}

\subsubsection{Effect of Return-to-go Credit Assignment}
\label{sec:ablation_lookahead_loss}

We next isolate the contribution of return-to-go credit assignment. Both variants use the
same logit-interpolant target $\tilde p_{\beta_k}$ and the same student-generated
trajectories, but differ in how the token-level signal is assigned.

The \emph{Local Token Gradient} variant uses only the instantaneous log-ratio:
\begin{equation}
    \widehat{g}_{\mathrm{local}}(y)
    =
    \sum_{t=1}^{T}
    \rho_t^{\beta_k}(y)
    \nabla_\theta
    \log \pi_\theta(y_t\mid x,y_{<t}).
    \label{eq:ablation_local_gradient}
\end{equation}
In contrast, the \emph{Return-to-go} variant uses the cumulative future mismatch:
\begin{equation}
    \widehat{g}_{\mathrm{RTG}}(y)
    =
    \sum_{t=1}^{T}
    G_{t,\gamma}^{\beta_k}(y)
    \nabla_\theta
    \log \pi_\theta(y_t\mid x,y_{<t}).
    \label{eq:ablation_rtg_gradient}
\end{equation}

As shown in Table~\ref{tab:lookahead_credit_ablation}, return-to-go credit assignment
consistently outperforms the local token-gradient objective. It improves avg@12 by
$1.12$ points on AIME 2024, $5.55$ points on AIME 2025, and $3.61$ points on
HMMT 2025. Since the two variants share the same target and sampling distribution, these
gains isolate the benefit of assigning earlier tokens credit for the future KL mismatch they
induce.

\input{tables/lad_ablation}

\subsubsection{Effect of Interpolant References}
\label{sec:ablation_interpolant_reference}

We ablate how the two endpoints of the logit interpolant are constructed. Let
$\theta_0$ denote the initial student parameters, let $\theta_k$ denote the student
parameters at training step $k$, and let
$\bar{\theta}_k=\operatorname{sg}(\theta_k)$ denote a stop-gradient copy of the current
student. In our self-distillation setting, the teacher branch uses the same model
architecture but additionally conditions on privileged information $c$.

We compare three reference choices. First, \emph{fixed student + fixed teacher} constructs
both endpoints from the initial model:
\begin{equation}
    \tilde p_{w}^{\mathrm{FF}}
    (\cdot\mid h_t,c)
    =
    \operatorname{softmax}
    \left(
    (1-w)z_{\theta_0}(\cdot\mid h_t)
    +
    w z_{\theta_0}(\cdot\mid h_t,c)
    \right).
    \label{eq:fixed_fixed_interpolant}
\end{equation}
Second, \emph{dynamic student + dynamic teacher} refreshes both endpoints from the
current model:
\begin{equation}
    \tilde p_{w}^{\mathrm{DD}}
    (\cdot\mid h_t,c)
    =
    \operatorname{softmax}
    \left(
    (1-w)z_{\bar{\theta}_k}(\cdot\mid h_t)
    +
    w z_{\bar{\theta}_k}(\cdot\mid h_t,c)
    \right).
    \label{eq:dynamic_dynamic_interpolant}
\end{equation}
Third, \emph{dynamic student + fixed teacher}, our default construction, uses the current
student as the student-side endpoint and a fixed privileged teacher as the teacher-side
endpoint:
\begin{equation}
    \tilde p_{w}^{\mathrm{DF}}
    (\cdot\mid h_t,c)
    =
    \operatorname{softmax}
    \left(
    (1-w)z_{\bar{\theta}_k}(\cdot\mid h_t)
    +
    w z_{\theta_0}(\cdot\mid h_t,c)
    \right).
    \label{eq:dynamic_fixed_interpolant}
\end{equation}

To isolate the effect of the interpolant references, we evaluate all three constructions using
the same return-to-go estimator and identical training settings. We keep the teacher weight
fixed at $w=0.5$, so that the two endpoints contribute equally to the target distribution.
As shown in Figure~\ref{fig:interpolant_reference_ablation}, the dynamic-student and
fixed-teacher construction performs best across all three benchmarks. These results suggest
that the interpolant benefits from tracking the evolving student distribution while retaining
a stable privileged teacher as its supervisory anchor.

\input{figures/interpolation_reference}

\subsubsection{Effect of the Interpolation Schedule}
\label{sec:ablation_interpolation_schedule}

Finally, we study how the interpolation schedule affects performance. We focus on linear
schedules for the teacher weight $w_k=1/\beta_k$, varying the start and end values in
Eq.~\ref{eq:exp_weight_schedule}. We also include a fixed $w_k=0.5$ baseline.

As shown in Table~\ref{tab:interpolation_schedule_ablation}, the schedule influences
performance across benchmarks. The $0.5\rightarrow0.8$ schedule gives the strongest
overall result among the schedules considered here, while other schedules can be better on
individual benchmarks. This suggests that the strength and timing of teacher guidance are
important practical choices. Beyond the linear schedules studied here, nonlinear,
piecewise, or performance-adaptive schedules may further improve performance. Determining
whether an optimal schedule exists, and how it depends on model scale, training budget,
and task distribution, is an interesting direction for future work.

\input{tables/schedule_ablation}

\begin{remark}
In addition to $\beta$-OPSD, which changes the distillation target through scheduled
logit interpolants, one can also construct guided rollouts by sampling from a mixture of the
student and the privileged teacher. This variant changes the trajectory distribution rather
than only the target distribution, and therefore requires importance correction. We provide
details and results in Appendix~\ref{app:proposal_mixture}.
\end{remark}

%% file: tables/qwen1.7b_eval.tex
\begin{table*}[t]
\centering
\small
\caption{
Avg@12 results for Qwen3 models evaluated at the 100-step checkpoint from the default
200-step training schedule. Gains denote absolute differences of \ours{} relative to
Vanilla OPSD within the same model scale. For \ours{}, we report results under linear
schedules along the target path.
}
\label{tab:qwen3_results}
\setlength{\tabcolsep}{9pt}
\renewcommand{\arraystretch}{1.08}
\begin{tabular}{@{}lcccc@{}}
\toprule
Method
& AIME 2024
& AIME 2025
& HMMT 2025
& Average \\
\midrule

\multicolumn{5}{@{}l}{\textit{Qwen3-8B}} \\
\quad Base model
& 73.06
& 67.78
& 29.72
& 56.85 \\
\quad Vanilla OPSD
& 75.83
& 66.67
& \textbf{33.06}
& 58.52 \\
\quad SFT
& 72.78
& 64.72
& 30.83
& 56.11 \\
\quad GRPO
& 75.83
& 68.61
& 31.11
& 58.52 \\
\rowcolor{gray!12}
\quad \textbf{\oursbold{} (ours)}
& \textbf{78.33} \posgain{+2.50}
& \textbf{70.28} \posgain{+3.61}
& 31.94 \neggain{-1.12}
& \textbf{60.18} \posgain{+1.66} \\

\midrule

\multicolumn{5}{@{}l}{\textit{Qwen3-4B}} \\
\quad Base model
& \textbf{77.78}
& 64.17
& 29.44
& 57.13 \\
\quad Vanilla OPSD
& 75.00
& 65.00
& 28.33
& 56.11 \\
\quad SFT
& 23.61
& 19.72
& 10.56
& 17.96 \\
\quad GRPO
& 76.39
& 64.17
& 30.28
& 56.95 \\
\rowcolor{gray!12}
\quad \textbf{\oursbold{} (ours)}
& 73.89 \neggain{-1.11}
& \textbf{69.17} \posgain{+4.17}
& \textbf{30.56} \posgain{+2.23}
& \textbf{57.87} \posgain{+1.76} \\

\midrule

\multicolumn{5}{@{}l}{\textit{Qwen3-1.7B}} \\
\quad Base model
& 50.00
& 37.22
& 14.44
& 33.89 \\
\quad Vanilla OPSD
& 44.17
& 35.56
& 13.33
& 31.02 \\
\quad SFT
& 37.50
& 29.17
& 12.22
& 26.30 \\
\quad GRPO
& 48.33
& 38.61
& 15.00
& 33.98 \\
\rowcolor{gray!12}
\quad \textbf{\oursbold{} (ours)}
& \textbf{53.33} \posgain{+9.16}
& \textbf{40.83} \posgain{+5.27}
& \textbf{16.11} \posgain{+2.78}
& \textbf{36.76} \posgain{+5.74} \\

\bottomrule
\end{tabular}
\end{table*}

%% file: sections/algorithm.tex
\begin{algorithm}[t]
\caption{\ours: Smooth Target Distillation Toward Logit Interpolants}
\label{alg:slide}
\begin{algorithmic}[1]
\Require Student $\pi_\theta$, privileged teacher $p_T$, dataset $\mathcal{D}$,
training budget $K$, schedule endpoints $w_{\mathrm{start}},w_{\mathrm{end}}$,
discount factor $\gamma$
\For{$k=0,\ldots,K-1$}
    \State Sample a minibatch $(x,c)\sim\mathcal{D}$
    \State Sample on-policy responses $y=(y_1,\ldots,y_T)\sim\pi_\theta(\cdot\mid x)$
    \State Set normalized progress $s_k=k/(K-1)$
    \State Set teacher weight $w_k$ by Eq.~\ref{eq:exp_weight_schedule} and $\beta_k=1/w_k$
    \State Set $\bar{\theta}\leftarrow \operatorname{sg}(\theta)$
    \State Construct interpolant targets $\tilde p_{\beta_k}(\cdot\mid h_t,c)$
    along the sampled prefixes using Eq.~\ref{eq:exp_interpolant_target}
    \State Compute token mismatches $\rho_t^{\beta_k}(y)$ using Eq.~\ref{eq:exp_token_mismatch}
    \State Compute discounted return-to-go weights $G_{t,\gamma}^{\beta_k}(y)$
    using Eq.~\ref{eq:exp_return_to_go}
    \State Minimize $\mathcal{L}_{\ours}(\theta;y)$ in Eq.~\ref{eq:exp_beta_opsd_loss}
    and update $\theta$
\EndFor
\end{algorithmic}
\end{algorithm}

%% file: tables/logit_interpolant_ablation.tex
\begin{table}[t]
\centering
\small
\caption{
Effect of the distillation target. Both variants use the same return-to-go
estimator. We report avg@12 (\%); the best result in each column is bolded. For the \ours{} target, we report the result with $w_{\mathrm{start}}=0.5$ and $w_{\mathrm{end}}=0.8$.
}
\label{tab:logit_interpolant_ablation}
\begin{tabular}{lccc}
\toprule
Method
& AIME 2024
& AIME 2025
& HMMT 2025 \\
\midrule
Teacher + Return-to-go
& 47.30
& 35.53
& 14.44 \\
\textbf{\oursbold{} target}
& \textbf{53.33}
& \textbf{40.83}
& \textbf{16.11} \\
\bottomrule
\end{tabular}
\end{table}

%% file: tables/lad_ablation.tex
\begin{table}[t]
\centering
\small
\caption{
Effect of return-to-go credit assignment. Both variants use the same logit-interpolant
target with fixed schedules $w_{k}=\frac{1}{\beta_k}=0.5$. We report avg@12 (\%).
}
\label{tab:lookahead_credit_ablation}
\begin{tabular}{lccc}
\toprule
Method
& AIME 2024
& AIME 2025
& HMMT 2025 \\
\midrule
Local Token Gradient& 49.44
& 32.78
& 13.06 \\
\textbf{Return-to-go Gradient}
& \textbf{50.56}
& \textbf{38.33}
& \textbf{16.67} \\
\bottomrule
\end{tabular}
\end{table}

%% file: figures/interpolation_reference.tex
\begin{figure}[t]
    \centering
    \includegraphics[width=\columnwidth]{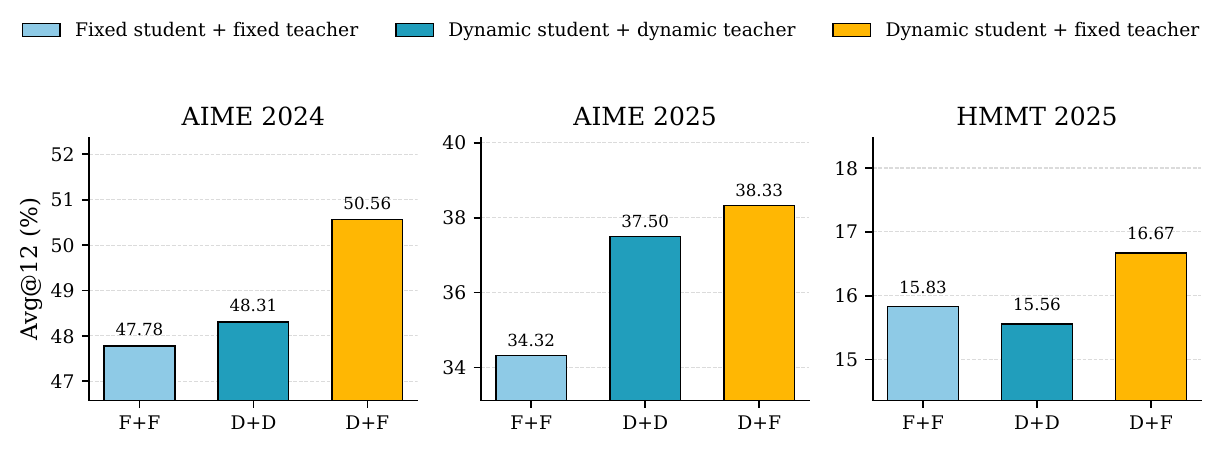}
    \caption{
    Effect of interpolant reference choices. All variants use the same return-to-go
    estimator and fixed teacher weight $w_k=1/\beta_k=0.5$. We report avg@12.
    Each subplot uses a benchmark-specific y-axis range to highlight relative differences.
    The dynamic student + fixed teacher construction performs best across all benchmarks.
    }
    \label{fig:interpolant_reference_ablation}
\end{figure}

%% file: tables/schedule_ablation.tex
\begin{table}[t]
\centering
\small
\caption{
Effect of interpolation schedules. The notation $a\rightarrow b$ indicates that the teacher-weight endpoints $w_{\mathrm{start}}=a$ and $w_{\mathrm{end}}=b$ are using the schedule in Eq~\ref{eq:exp_weight_schedule}. We report avg@12 (\%).
}
\label{tab:interpolation_schedule_ablation}
\begin{tabular}{lccc}
\toprule
Schedule
& AIME 2024
& AIME 2025
& HMMT 2025 \\
\midrule
Fixed $0.5$
& 50.56
& 38.33
& \textbf{16.67} \\

Linear $0.2\rightarrow0.8$
& 51.39
& 37.50
& 16.39 \\

Linear $0.8\rightarrow0.2$
& 51.94
& 38.06
& 15.83 \\

Linear $0.5\rightarrow0.8$
& 53.33
& \textbf{40.83}
& 16.11 \\

Linear $0.8\rightarrow0.5$
& \textbf{54.72}
& 38.61
& 15.83 \\
\bottomrule
\end{tabular}
\end{table}

%% file: sections/conclusion.tex
\section{Conclusion}
\label{sec:conclusion}
We introduced \ours{}, a framework built on the view that on-policy distillation is
trajectory-level policy optimization rather than merely local token imitation. Under this
view, standard sequence-level OPSD is the $\beta=1$ teacher endpoint of a
reference-regularized RL family. In the conservative regime, the optimal policies form a
geometric path from a reference student to the privileged teacher, motivating tractable
token-level logit interpolants instead of direct teacher matching throughout training.

\ours{} further replaces the myopic local token-KL update with a return-to-go credit assignment that
accounts for future distributional mismatch and provides an unbiased gradient estimator for
the sequence-level objective. Experiments on competition-level mathematical reasoning
benchmarks show consistent improvements over vanilla OPSD, SFT, and GRPO methods. Our
ablations confirm the importance of both interpolant targets and future-aware credit
assignment. These results suggest that smooth target paths and trajectory-level credit
assignment are effective principles for on-policy distillation. Exploring adaptive or
theoretically optimized interpolation schedules, lower-variance return-to-go estimators, and
applications beyond mathematical reasoning are promising directions for future work.

%% file: sections/acknowledge.tex
\section{Acknowledgment}

Xu, Liu and Huang are supported by DARPA HR001124S0029-AIQ-FP-019, National Science Foundation TRAILS Institute (2229885). Private support was provided by Open Philanthropy and Apple. The Authors acknowledge the National Artificial Intelligence Research Resource (NAIRR) Pilot for contributing to this research result.

%% file: sections/appendix.tex
\newpage

\section{Notation Reference}
\label{app:notation}
\input{sections/notation_table}
\clearpage

\section{Theoretical Proofs}
\label{app:rtg-proofs}
\subsection{Proof of Proposition~\ref{prop:opsd_beta_one}}
\label{app:regularized_rl}
\begin{proof}
Fix a prompt $x$ and privileged context $c$. For readability, we suppress the conditioning
on $(x,c)$ when it is unambiguous. Expanding Eq.~\ref{eq:beta_opsd_objective}, we have
\begin{align}
    \mathcal{J}_{\beta}(\pi_\theta)
    &=
    \mathbb{E}_{y\sim\pi_\theta}
    \left[
    \log
    \frac{
    p_T(y)
    }{
    \pi_{\mathrm{ref}}(y)
    }
    \right]
    -
    \beta
    D_{\mathrm{KL}}
    \left(
    \pi_\theta
    \,\middle\|\,
    \pi_{\mathrm{ref}}
    \right) \\
    &=
    \mathbb{E}_{y\sim\pi_\theta}
    \left[
    \log p_T(y)
    -
    \log \pi_{\mathrm{ref}}(y)
    \right]
    -
    \beta
    \mathbb{E}_{y\sim\pi_\theta}
    \left[
    \log \pi_\theta(y)
    -
    \log \pi_{\mathrm{ref}}(y)
    \right] \\
    &=
    \mathbb{E}_{y\sim\pi_\theta}
    \left[
    \log p_T(y)
    -
    \beta \log \pi_\theta(y)
    +
    (\beta-1)\log \pi_{\mathrm{ref}}(y)
    \right].
    \label{eq:beta_opsd_expanded}
\end{align}
On the other hand,
\begin{align}
    &-
    D_{\mathrm{KL}}
    \left(
    \pi_\theta
    \,\middle\|\,
    p_T
    \right)
    -
    (\beta-1)
    D_{\mathrm{KL}}
    \left(
    \pi_\theta
    \,\middle\|\,
    \pi_{\mathrm{ref}}
    \right) \\
    &=
    -
    \mathbb{E}_{y\sim\pi_\theta}
    \left[
    \log \pi_\theta(y)
    -
    \log p_T(y)
    \right]
    -
    (\beta-1)
    \mathbb{E}_{y\sim\pi_\theta}
    \left[
    \log \pi_\theta(y)
    -
    \log \pi_{\mathrm{ref}}(y)
    \right] \\
    &=
    \mathbb{E}_{y\sim\pi_\theta}
    \left[
    \log p_T(y)
    -
    \beta \log \pi_\theta(y)
    +
    (\beta-1)\log \pi_{\mathrm{ref}}(y)
    \right].
\end{align}
Comparing this expression with Eq.~\ref{eq:beta_opsd_expanded}, we obtain the equivalent
decomposition
\begin{equation}
    \mathcal{J}_{\beta}(\pi_\theta)
    =
    -
    D_{\mathrm{KL}}
    \left(
    \pi_\theta(\cdot\mid x)
    \,\middle\|\,
    p_T(\cdot\mid x,c)
    \right)
    -
    (\beta-1)
    D_{\mathrm{KL}}
    \left(
    \pi_\theta(\cdot\mid x)
    \,\middle\|\,
    \pi_{\mathrm{ref}}(\cdot\mid x)
    \right).
    \label{eq:beta_opsd_decomposition}
\end{equation}
Setting $\beta=1$ cancels the reference-regularization term:
\begin{equation}
    \mathcal{J}_{1}(\pi_\theta)
    =
    -
    D_{\mathrm{KL}}
    \left(
    \pi_\theta(\cdot\mid x)
    \,\middle\|\,
    p_T(\cdot\mid x,c)
    \right).
\end{equation}
Therefore,
\begin{equation}
    \max_{\theta}\mathcal{J}_{1}(\pi_\theta)
    \quad
    \Longleftrightarrow
    \quad
    \min_{\theta}
    D_{\mathrm{KL}}
    \left(
    \pi_\theta(\cdot\mid x)
    \,\middle\|\,
    p_T(\cdot\mid x,c)
    \right),
\end{equation}
which is exactly the standard sequence-level OPSD objective in
Eq.~\ref{eq:standard_opsd_kl}. This proves the proposition.
\end{proof}

\subsection{Proof of Proposition~\ref{prop:geometric_optimal_policy}}
\label{app:geometric-optimal-policy}
\begin{proof}
The $\beta$-OPSD objective:
\begin{equation}
    \mathcal{J}_{\beta}(\pi)
    =
    \mathbb{E}_{y\sim\pi(\cdot\mid x)}
    \left[
    \log
    \frac{
    p_T(y\mid x,c)
    }{
    \pi_{\mathrm{ref}}(y\mid x)
    }
    \right]
    -
    \beta
    D_{\mathrm{KL}}
    \left(
    \pi(\cdot\mid x)
    \,\middle\|\,
    \pi_{\mathrm{ref}}(\cdot\mid x)
    \right).
\end{equation}
Expanding the KL term gives
\begin{align}
    \mathcal{J}_{\beta}(\pi)
    &=
    \mathbb{E}_{y\sim\pi}
    \left[
    \log p_T(y)
    -
    \log \pi_{\mathrm{ref}}(y)
    \right]
    -
    \beta
    \mathbb{E}_{y\sim\pi}
    \left[
    \log \pi(y)
    -
    \log \pi_{\mathrm{ref}}(y)
    \right] \\
    &=
    \mathbb{E}_{y\sim\pi}
    \left[
    \log p_T(y)
    -
    \beta \log \pi(y)
    +
    (\beta-1)\log \pi_{\mathrm{ref}}(y)
    \right].
    \label{eq:beta_opsd_expanded_for_optimal_policy}
\end{align}

Now define
\begin{equation}
    \pi_{\beta}^{\star}(y\mid x,c)
    =
    \frac{
    \pi_{\mathrm{ref}}(y\mid x)^{1-\frac{1}{\beta}}
    p_T(y\mid x,c)^{\frac{1}{\beta}}
    }{
    Z_{\beta}(x,c)
    },
\end{equation}
where
\begin{equation}
    Z_{\beta}(x,c)
    =
    \sum_{y'}
    \pi_{\mathrm{ref}}(y'\mid x)^{1-\frac{1}{\beta}}
    p_T(y'\mid x,c)^{\frac{1}{\beta}}.
\end{equation}
Taking logs, we have
\begin{equation}
    \log \pi_{\beta}^{\star}(y\mid x,c)
    =
    \left(1-\frac{1}{\beta}\right)
    \log \pi_{\mathrm{ref}}(y\mid x)
    +
    \frac{1}{\beta}
    \log p_T(y\mid x,c)
    -
    \log Z_{\beta}(x,c).
\end{equation}
Multiplying both sides by $\beta$ gives
\begin{equation}
    \beta \log \pi_{\beta}^{\star}(y\mid x,c)
    =
    (\beta-1)
    \log \pi_{\mathrm{ref}}(y\mid x)
    +
    \log p_T(y\mid x,c)
    -
    \beta \log Z_{\beta}(x,c).
    \label{eq:beta_log_optimal_policy}
\end{equation}
Substituting Eq.~\ref{eq:beta_log_optimal_policy} into
Eq.~\ref{eq:beta_opsd_expanded_for_optimal_policy}, we obtain
\begin{align}
    \mathcal{J}_{\beta}(\pi)
    &=
    \mathbb{E}_{y\sim\pi}
    \left[
    \beta \log \pi_{\beta}^{\star}(y\mid x,c)
    +
    \beta \log Z_{\beta}(x,c)
    -
    \beta \log \pi(y\mid x)
    \right] \\
    &=
    -\beta
    \mathbb{E}_{y\sim\pi}
    \left[
    \log
    \frac{
    \pi(y\mid x)
    }{
    \pi_{\beta}^{\star}(y\mid x,c)
    }
    \right]
    +
    \beta \log Z_{\beta}(x,c) \\
    &=
    -\beta
    D_{\mathrm{KL}}
    \left(
    \pi(\cdot\mid x)
    \,\middle\|\,
    \pi_{\beta}^{\star}(\cdot\mid x,c)
    \right)
    +
    \beta \log Z_{\beta}(x,c).
    \label{eq:beta_opsd_as_negative_kl}
\end{align}
The second term is independent of $\pi$, while the KL term is nonnegative and is minimized
to zero if and only if
\begin{equation}
    \pi(\cdot\mid x)
    =
    \pi_{\beta}^{\star}(\cdot\mid x,c).
\end{equation}
Therefore, $\pi_{\beta}^{\star}$ is the maximizer of
Eq.~\ref{eq:beta_opsd_objective}. This proves the proposition.
\end{proof}

\subsection{Proof of the Return-to-Go Gradient Estimator}
\label{app:onpolicy-lookahead-gradient}
\begin{proof}
We prove that when $\gamma=1$, the return-to-go estimator in
Eq.~\ref{eq:slide_return_to_go} gives an unbiased estimator of the sequence-level
reverse-KL gradient.

\begin{proposition}[Unbiased sequence-level gradient estimator]
\label{prop:rtg_unbiased_beta_opsd}
Fix a training step $k$ and treat the interpolant target
$\tilde p_{\beta_k}(\cdot\mid x,c)$ as fixed during the current gradient update. Define
\begin{equation}
    \mathcal{L}_{\beta_k}(\theta)
    =
    D_{\mathrm{KL}}
    \left(
    \pi_\theta(\cdot\mid x)
    \,\middle\|\,
    \tilde p_{\beta_k}(\cdot\mid x,c)
    \right),
\end{equation}
and let
\begin{equation}
    \rho_t^{\beta_k}(y)
    =
    \log \pi_\theta(y_t\mid x,y_{<t})
    -
    \log \tilde p_{\beta_k}(y_t\mid x,y_{<t},c).
\end{equation}
For $\gamma=1$, define
\begin{equation}
    G_t^{\beta_k}(y)
    =
    \sum_{s=t}^{T}
    \rho_s^{\beta_k}(y).
\end{equation}
Then
\begin{equation}
    \widehat{g}_{\mathrm{RTG}}(y)
    =
    \sum_{t=1}^{T}
    G_t^{\beta_k}(y)
    \nabla_\theta
    \log \pi_\theta(y_t\mid x,y_{<t})
\end{equation}
satisfies
\begin{equation}
    \mathbb{E}_{y\sim \pi_\theta(\cdot\mid x)}
    \left[
    \widehat{g}_{\mathrm{RTG}}(y)
    \right]
    =
    \nabla_\theta
    \mathcal{L}_{\beta_k}(\theta).
\end{equation}
\end{proposition}

\begin{proof}
For readability, fix $x,c,\beta_k$ and suppress them when clear. By the autoregressive
factorization,
\begin{equation}
    \pi_\theta(y\mid x)
    =
    \prod_{t=1}^{T}
    \pi_\theta(y_t\mid x,y_{<t}),
    \qquad
    \tilde p_{\beta_k}(y\mid x,c)
    =
    \prod_{t=1}^{T}
    \tilde p_{\beta_k}(y_t\mid x,y_{<t},c).
\end{equation}
Therefore,
\begin{equation}
    \mathcal{L}_{\beta_k}(\theta)
    =
    \mathbb{E}_{y\sim\pi_\theta(\cdot\mid x)}
    \left[
    \sum_{t=1}^{T}
    \rho_t^{\beta_k}(y)
    \right].
\end{equation}
Let
\begin{equation}
    R^{\beta_k}(y)
    :=
    \sum_{t=1}^{T}
    \rho_t^{\beta_k}(y),
    \qquad
    g_t(y)
    :=
    \nabla_\theta
    \log \pi_\theta(y_t\mid x,y_{<t}).
\end{equation}
Then
\begin{equation}
    \mathcal{L}_{\beta_k}(\theta)
    =
    \mathbb{E}_{y\sim\pi_\theta}
    \left[
    R^{\beta_k}(y)
    \right].
\end{equation}
Differentiating with the score-function identity gives
\begin{align}
    \nabla_\theta \mathcal{L}_{\beta_k}(\theta)
    &=
    \mathbb{E}_{y\sim\pi_\theta}
    \left[
    R^{\beta_k}(y)
    \nabla_\theta \log \pi_\theta(y\mid x)
    \right]
    +
    \mathbb{E}_{y\sim\pi_\theta}
    \left[
    \nabla_\theta R^{\beta_k}(y)
    \right].
\end{align}
Since
\begin{equation}
    \nabla_\theta \log \pi_\theta(y\mid x)
    =
    \sum_{t=1}^{T}
    g_t(y),
\end{equation}
and the interpolant target $\tilde p_{\beta_k}$ is fixed during the update,
\begin{equation}
    \nabla_\theta R^{\beta_k}(y)
    =
    \sum_{t=1}^{T}
    g_t(y).
\end{equation}
The expected score is zero:
\begin{equation}
    \mathbb{E}_{y\sim\pi_\theta}
    \left[
    \sum_{t=1}^{T}
    g_t(y)
    \right]
    =
    0.
\end{equation}
Thus,
\begin{align}
    \nabla_\theta \mathcal{L}_{\beta_k}(\theta)
    &=
    \mathbb{E}_{y\sim\pi_\theta}
    \left[
    R^{\beta_k}(y)
    \sum_{t=1}^{T}
    g_t(y)
    \right] \\
    &=
    \mathbb{E}_{y\sim\pi_\theta}
    \left[
    \sum_{t=1}^{T}
    g_t(y)
    \sum_{s=1}^{T}
    \rho_s^{\beta_k}(y)
    \right].
    \label{eq:full_return_gradient_beta}
\end{align}

We now remove the past terms by causality. For $s<t$, the quantity
$\rho_s^{\beta_k}(y)$ depends only on earlier tokens and is fixed after conditioning on the
prefix $(x,y_{<t})$. Meanwhile,
\begin{align}
    \mathbb{E}_{y_t\sim\pi_\theta(\cdot\mid x,y_{<t})}
    \left[
    g_t(y)
    \mid x,y_{<t}
    \right]
    &=
    \sum_{v\in\mathcal{V}}
    \pi_\theta(v\mid x,y_{<t})
    \nabla_\theta
    \log \pi_\theta(v\mid x,y_{<t}) \\
    &=
    \sum_{v\in\mathcal{V}}
    \nabla_\theta
    \pi_\theta(v\mid x,y_{<t}) \\
    &=
    \nabla_\theta
    \sum_{v\in\mathcal{V}}
    \pi_\theta(v\mid x,y_{<t}) \\
    &=
    0.
\end{align}
Therefore,
\begin{equation}
    \mathbb{E}_{y\sim\pi_\theta}
    \left[
    g_t(y)
    \sum_{s<t}
    \rho_s^{\beta_k}(y)
    \right]
    =
    0.
\end{equation}
Applying this to Eq.~\ref{eq:full_return_gradient_beta}, only the current and future terms
remain:
\begin{align}
    \nabla_\theta \mathcal{L}_{\beta_k}(\theta)
    &=
    \mathbb{E}_{y\sim\pi_\theta}
    \left[
    \sum_{t=1}^{T}
    g_t(y)
    \sum_{s=t}^{T}
    \rho_s^{\beta_k}(y)
    \right] \\
    &=
    \mathbb{E}_{y\sim\pi_\theta}
    \left[
    \sum_{t=1}^{T}
    G_t^{\beta_k}(y)
    g_t(y)
    \right].
\end{align}
This proves that the return-to-go estimator is unbiased for the sequence-level gradient.
\end{proof}

Finally, the practical surrogate in Eq.~\ref{eq:slide_loss} with $\gamma=1$ has gradient
\begin{equation}
    \nabla_\theta \mathcal{L}_{\ours{}}(\theta;y)
    =
    \frac{1}{T}
    \sum_{t=1}^{T}
    G_t^{\beta_k}(y)
    \nabla_\theta
    \log \pi_\theta(y_t\mid x,y_{<t}),
\end{equation}
where the return weights are detached. Hence, its expectation equals
$\frac{1}{T}\nabla_\theta \mathcal{L}_{\beta_k}(\theta)$. The factor $1/T$ is only a
normalization constant. For $\gamma<1$, the same loss becomes a discounted practical
approximation that reduces the magnitude and variance of long-horizon returns, but the
exact unbiasedness statement corresponds to $\gamma=1$.
\end{proof}

\section{Experimental Details}
\label{sec:exp_details}
\subsection{Training Details}
\label{sec:training_detail}
\paragraph{Shared training configuration.}
Vanilla OPSD and \ours{} use the same model, optimization settings, rollout configuration,
and computational budget. We train Qwen3-1.7B on four NVIDIA RTX A6000 GPUs using
four-way data parallelism with \texttt{accelerate}. Each training rank runs a colocated
single-GPU vLLM instance for on-policy generation. Training is run for a maximum of
$200$ optimization steps with an eight-hour wall-time limit.

\paragraph{Optimization and parameter-efficient tuning.}
We use a learning rate of $5\times10^{-6}$ and clip the gradient norm at $0.1$.
The per-device batch size is $1$, with $8$ gradient-accumulation steps, giving an effective
batch size of
\begin{equation}
    4\ \text{GPUs}
    \times 1\ \text{sample per GPU}
    \times 8\ \text{accumulation steps}
    = 32.
\end{equation}
Gradient checkpointing is enabled. All models are trained in bfloat16 precision using
FlashAttention-2. We use LoRA with rank $r=64$ and scaling parameter
$\alpha_{\mathrm{LoRA}}=128$. LoRA adapters are applied to
\texttt{q\_proj}, \texttt{k\_proj}, \texttt{v\_proj}, \texttt{o\_proj},
\texttt{gate\_proj}, \texttt{up\_proj}, and \texttt{down\_proj}.

\paragraph{On-policy generation.}
Training trajectories are generated with a maximum sequence length of $20{,}000$ tokens
and a maximum completion length of $1{,}024$ tokens. We sample with temperature $1.1$,
top-$p=0.95$, and top-$k=20$. Generation uses vLLM in colocated mode with tensor-parallel
size $1$ and GPU-memory utilization set to $0.5$. Unless otherwise stated, these settings
are shared by Vanilla OPSD and \ours{}, ensuring that their comparison isolates the effect of
the distillation objective and Look-Ahead credit assignment.

\subsection{Evaluation Details}
\label{sec:eval_detail}
\paragraph{Evaluation pipeline.}
We evaluate the Qwen3-1.7B, Qwen3-4B and Qwen3-8B checkpoints using their corresponding
LoRA adapters loaded on top of the original Qwen3 base models.
We evaluate on AIME 2024, AIME 2025, and HMMT 2025.

\paragraph{Decoding configuration.}
We use stochastic decoding with thinking mode enabled. For each problem, we generate
$k=12$ independent candidate solutions using temperature $0.6$, top-$p$ sampling with
$p=0.95$, and top-$k=50$. The evaluation batch size is $32$. We set
\texttt{max\_new\_tokens=0}, corresponding to no separate completion-length cap in the
evaluation script, and use the default maximum model length of $40{,}960$ tokens in
thinking mode. Unless explicitly changed for an ablation, all reported results use this
same decoding configuration.

\paragraph{Scoring and metrics.}
We report
\emph{pass@12}, the fraction of problems for which at least one of the 12 sampled
solutions is correct, and \emph{avg@12}, the average correctness across all 12 sampled
solutions. The evaluation code returns both metrics as raw fractions; throughout the paper,
we multiply them by $100$ and report percentage points.

\paragraph{Inference configuration.}
Evaluation uses tensor-parallel size $1$, GPU memory utilization $0.9$, and a single
evaluation repeat. LoRA checkpoints are evaluated without modifying the underlying base
model weights. Checkpoint sweeps evaluate training steps
$\{50,75,100,200\}$, with one Slurm evaluation job submitted for each
checkpoint--dataset pair.

\section{Sampling from Logit Interpolants}
\label{app:proposal_mixture}
In addition to student on-policy sampling, we evaluate a guided sampling variant that
constructs trajectories from a mixture of the student and privileged teacher. For proposal-mixture weight $\eta\in[0,1]$, we define the token-level proposal distribution
\begin{equation}
    m_{\bar{\theta},\eta}(\cdot\mid h_t,c)
    =
    (1-\eta)\,
    \pi_{\bar{\theta}}(\cdot\mid h_t)
    +
    \eta\,
    p_T(\cdot\mid h_t,c),
    \label{eq:student_teacher_mixed_sampler}
\end{equation}
where $\pi_{\bar{\theta}}$ is a stop-gradient copy of the current student. Thus,
$\eta=0$ recovers student on-policy sampling, whereas $\eta=1$ samples entirely from
the privileged teacher. The resulting autoregressive proposal is
\begin{equation}
    m_{\bar{\theta},\eta}(y\mid x,c)
    =
    \prod_{t=1}^{T}
    m_{\bar{\theta},\eta}(y_t\mid x,y_{<t},c).
    \label{eq:student_teacher_mixed_sequence}
\end{equation}

Because the objective in Eq.~\ref{eq:standard_opsd_kl} is defined under the student trajectory
distribution $\pi_\theta$, sampling from $m_{\bar{\theta},\eta}$ introduces a distribution
mismatch. We correct this mismatch using per-decision importance sampling. For a trajectory
$y\sim m_{\bar{\theta},\eta}(\cdot\mid x,c)$, we define
\begin{equation}
    w_t(y)
    =
    \frac{
    \pi_\theta(y_t\mid x,y_{<t})
    }{
    m_{\bar{\theta},\eta}(y_t\mid x,y_{<t},c)
    },
    \qquad
    W_t(y)
    =
    \prod_{i=1}^{t} w_i(y).
    \label{eq:mixed_sampling_importance_ratio}
\end{equation}
The importance-weighted Look-Ahead return is
\begin{equation}
    G_{t,\gamma}^{\eta,\mathrm{mix}}(y)
    =
    \sum_{s=t}^{T}
    \gamma^{s-t}
    W_s(y)\rho_s^\tau(y),
    \label{eq:mixed_sampling_lookahead_return}
\end{equation}
where $\rho_s^\tau$ measures the student mismatch with the logit-interpolant target. We optimize the corresponding
surrogate
\begin{equation}
    \mathcal{L}_{\mathrm{mix\text{-}\ours}}(\theta;y)
    =
    \frac{1}{T}
    \sum_{t=1}^{T}
    \operatorname{sg}
    \left(
    G_{t,\gamma}^{\tau,\mathrm{mix}}(y)
    \right)
    \log \pi_\theta(y_t\mid x,y_{<t}).
    \label{eq:mixed_sampling_lad_loss}
\end{equation}
We note that this variant is inherently off-policy and requires combining the student and
teacher distributions at every decoding step. At the time of writing, such dual-model sampling is \textbf{not} directly supported by standard vLLM decoding pipelines, resulting in additional (nearly double)
implementation complexity and generation overhead. Although mixed student--teacher
sampling yields comparable improvements over vanilla OPSD, its reduced
sampling efficiency makes it less attractive in practice. 

\paragraph{Results}
We evaluate fixed proposal-mixture weights
$\eta\in\{0.2,0.5,0.8\}$ and linear schedules that either increase or decrease $\eta$
during training. Here, $\eta$ denotes the teacher weight in
Eq.~\ref{eq:student_teacher_mixed_sampler}, while $1-\eta$ denotes the student weight.
All mixed-sampling variants use \ours{} with discount factor $\gamma=0.99$ and
per-decision importance sampling (IS).

\input{tables/app_mix_sampling}

Table~\ref{tab:mixed_sampling_results} shows that mixed student--teacher sampling can
produce improvements comparable to those of the main on-policy \ours{} method. At 200
steps, every evaluated schedule improves avg@12 over Vanilla OPSD on both AIME
benchmarks. Across all checkpoints, the best observed gains over Vanilla OPSD are $8.06$
points on AIME 2024, $5.55$ points on AIME 2025, and $3.06$ points on HMMT 2025.
However, no single proposal schedule dominates across all benchmarks, indicating that the
off-policy variant is sensitive to both the mixture coefficient and the training checkpoint.
Given its additional dual-model decoding cost and the variance introduced by importance
sampling, we use student on-policy sampling as the default implementation of \ours{}.

%% file: sections/notation_table.tex
\begin{table}[hbt]
\centering
\small
\caption{
Summary of notation. For readability, we omit conditioning variables such as $(x,c)$
when they are clear from context. Throughout the paper, $\pi_\theta$ denotes the student
policy, $\pi_{\mathrm{ref}}$ denotes the reference policy, and $p_T$ denotes the privileged
teacher.
}
\label{tab:notation}
\begin{tabular}{p{0.28\linewidth}p{0.66\linewidth}}
\toprule
\textbf{Notation} & \textbf{Description} \\
\midrule

$x$ & Input prompt. \\
$c$ & Privileged context available only during training. \\
$y=(y_1,\ldots,y_T)$ & Generated reasoning trajectory. \\
$y_t$ & Token at decoding step $t$. \\
$y_{<t}$ & Prefix before token $t$. \\
$h_t=(x,y_{<t})$ & Decoding history or prefix state at step $t$. \\

\midrule

$\pi_\theta(y\mid x)$ & Student policy parameterized by $\theta$. \\
$p_T(y\mid x,c)$ & Privileged teacher policy conditioned on prompt $x$ and context $c$. \\
$\pi_{\mathrm{ref}}(y\mid x)$ & Reference policy, e.g., the initial student or a stop-gradient copy. \\
$\bar{\theta}$ & Stop-gradient copy of the current student parameters. \\

\midrule

$R(y;x,c)$ & Trajectory-level reward in the KL-regularized RL formulation. \\
$\mathcal{J}_{\beta}(\theta)$ & $\beta$-OPSD objective, obtained by using the teacher-to-reference log-ratio reward. \\
$\beta$ & KL regularization coefficient in $\beta$-OPSD. \\
$\beta_k$ & Regularization coefficient used at training step $k$. \\
$\pi_{\beta}^{\star}(y\mid x,c)$ & Closed-form optimal policy of the $\beta$-OPSD objective. \\

\midrule

$z_{\mathrm{ref}}(\cdot\mid h_t)$ & Reference/student logits at prefix $h_t$. \\
$z_T(\cdot\mid h_t,c)$ & Privileged teacher logits at prefix $h_t$. \\
$\tilde p_{\beta_k}(\cdot\mid h_t,c)$ & Local logit-interpolant target at step $k$. \\
$\tilde p_{\beta_k}(y\mid x,c)$ & Autoregressive sequence-level logit-interpolant target. \\
$s_k=\frac{k}{K-1}$ & Normalized training progress. \\
$w_{k}$ & Implementation notation for the teacher weight. \\

\midrule

$\rho_t^{\beta_k}(y)$ & Token-level log-ratio between the student and interpolant target. \\
$G_{t,\gamma}^{\beta_k}(y)$ & Discounted return-to-go from token $t$ using future interpolant log-ratios. \\
$\gamma$ & Discount factor for the return-to-go. \\
$\operatorname{sg}(\cdot)$ & Stop-gradient operator. \\

\midrule

$D_{\mathrm{KL}}(\cdot\|\cdot)$ & Kullback--Leibler divergence. \\
$\mathbb{E}_{y\sim\pi_\theta}$ & Expectation over trajectories sampled from the student policy. \\

\bottomrule
\end{tabular}
\end{table}

%% file: tables/app_mix_sampling.tex
\begin{table*}[htb]
\centering
\scriptsize
\caption{
Mixed student--teacher sampling results. All non-baseline variants use
Sampling from Logit Interpolants with discount factor $\gamma=0.99$ and
importance sampling (IS). Here, $\eta$ denotes the teacher weight in the
proposal distribution. We report pass@12 and avg@12 in percentage points.
}
\label{tab:mixed_sampling_results}
\setlength{\tabcolsep}{4pt}
\renewcommand{\arraystretch}{1.06}

\begin{tabular}{lccccccc}
\toprule
Proposal schedule
& Step
& \multicolumn{2}{c}{AIME 2024}
& \multicolumn{2}{c}{AIME 2025}
& \multicolumn{2}{c}{HMMT 2025} \\
\cmidrule(lr){3-4}
\cmidrule(lr){5-6}
\cmidrule(lr){7-8}
&
& pass@12
& avg@12
& pass@12
& avg@12
& pass@12
& avg@12 \\
\midrule

Vanilla OPSD
& 200
& 76.67
& 46.11
& 63.33
& 35.56
& 43.33
& 15.83 \\

\midrule
\multicolumn{8}{l}{\textit{Fixed proposal: $\eta=0.2$}} \\
\quad Mixed Sampling
& 50
& 76.67
& 50.83
& 60.00
& 37.50
& 40.00
& 15.83 \\
\quad Mixed Sampling
& 75
& 80.00
& 50.00
& 66.67
& 38.06
& 40.00
& 14.44 \\
\quad Mixed Sampling
& 100
& 76.67
& 52.78
& 53.33
& 35.83
& 33.33
& 13.61 \\
\quad Mixed Sampling
& 200
& 73.33
& 51.94
& 63.33
& 37.22
& 43.33
& 16.11 \\

\midrule
\multicolumn{8}{l}{\textit{Fixed proposal: $\eta=0.5$}} \\
\quad Mixed Sampling
& 50
& 80.00
& 50.83
& 66.67
& 36.39
& 43.33
& 16.11 \\
\quad Mixed Sampling
& 75
& 80.00
& 52.22
& 66.67
& 37.78
& 46.67
& 16.11 \\
\quad Mixed Sampling
& 100
& 76.67
& \textbf{54.17}
& 63.33
& 39.44
& 46.67
& 17.22 \\
\quad Mixed Sampling
& 200
& 76.67
& 53.33
& 63.33
& 38.06
& 36.67
& 15.83 \\

\midrule
\multicolumn{8}{l}{\textit{Fixed proposal: $\eta=0.8$}} \\
\quad Mixed Sampling
& 50
& 80.00
& 50.83
& 60.00
& 36.94
& 43.33
& 15.83 \\
\quad Mixed Sampling
& 75
& 80.00
& 51.39
& 70.00
& \textbf{41.11}
& 46.67
& 16.11 \\
\quad Mixed Sampling
& 100
& 73.33
& 51.94
& 60.00
& 36.11
& 40.00
& 15.83 \\
\quad Mixed Sampling
& 200
& 73.33
& 51.39
& 66.67
& 38.61
& 36.67
& 13.89 \\

\midrule
\multicolumn{8}{l}{
\textit{Linear proposal schedule: $\eta:0.2\rightarrow0.8$}
} \\
\quad Mixed Sampling
& 50
& 80.00
& 50.83
& 70.00
& 38.06
& 36.67
& 13.61 \\
\quad Mixed Sampling
& 75
& 76.67
& 50.56
& 70.00
& 36.39
& \textbf{53.33}
& 18.06 \\
\quad Mixed Sampling
& 100
& 73.33
& 51.94
& 66.67
& 36.67
& 36.67
& 16.39 \\
\quad Mixed Sampling
& 200
& 76.67
& 50.56
& 60.00
& \textbf{41.11}
& 43.33
& 16.94 \\

\midrule
\multicolumn{8}{l}{
\textit{Linear proposal schedule: $\eta:0.8\rightarrow0.2$}
} \\
\quad Mixed Sampling
& 50
& 76.67
& 50.56
& \textbf{76.67}
& 36.39
& 30.00
& 13.61 \\
\quad Mixed Sampling
& 75
& 73.33
& 47.78
& 73.33
& 36.39
& 50.00
& \textbf{18.89} \\
\quad Mixed Sampling
& 100
& 80.00
& 50.56
& 60.00
& 38.61
& 36.67
& 15.83 \\
\quad Mixed Sampling
& 200
& 76.67
& 52.78
& 60.00
& 39.72
& 36.67
& 15.56 \\

\bottomrule
\end{tabular}
\end{table*}